%% file: main.tex
\pdfoutput=1
\documentclass[10pt,a4paper]{article}
\input{preamble.tex}

\title{\textbf{Permutation-based Reasoning and Intelligence Search Method (PRISM)}\\
\large Predicting When Search Helps with a Pre-Flight Protocol,\\[-0.1em]
and an Exhaustive Map of Language-Model Instruction Order}
\author{Blessings Mambwe\\[-0.1em]\small ML Collective}
\date{August 8, 2026}

\begin{document}
\maketitle
\vspace{-1.2em}

\input{sections/00_abstract.tex}
\input{sections/01_introduction.tex}
\input{sections/01_terminology.tex}
\input{sections/02_protocol.tex}
\input{sections/02_related_work.tex}
\input{sections/03_landscapes.tex}
\input{sections/04_falsification.tex}
\input{sections/05_llm_landscape.tex}
\input{sections/06_transfer_baselines.tex}
\input{sections/07_cross_domain.tex}
\input{sections/08_limitations.tex}
\input{sections/09_conclusion.tex}
\raggedbottom
\input{sections/08_applications.tex}

\smallskip
\noindent\textbf{Acknowledgments.}
The author gratefully acknowledges God for guidance and strength, and thanks Jake Beck for
reviewing the manuscript and providing constructive feedback. This work was conducted within the
ML Collective research community. The author retains responsibility for the analyses and
conclusions.

\label{maintext:end}
\clearpage
\flushbottom
\printbibliography

\clearpage
\appendix
\input{sections/a_theory.tex}
\input{sections/b_additional_results.tex}
\input{sections/c_reproducibility.tex}

\end{document}

%% file: preamble.tex
\usepackage[margin=2.25cm,top=2.15cm,bottom=2.25cm]{geometry}
\usepackage[T1]{fontenc}
\usepackage{lmodern,microtype,mathtools,amssymb,bm,amsthm}
\newtheorem{theorem}{Theorem}[section]
\newtheorem{proposition}[theorem]{Proposition}

\usepackage[ruled,vlined,linesnumbered]{algorithm2e}
\SetKwInput{KwIn}{Input}
\SetKwInput{KwOut}{Output}
\DontPrintSemicolon

\usepackage{float,newfloat,graphicx,xcolor,tikz,pgfplots,subcaption,booktabs,tabularx,enumitem}
\DeclareFloatingEnvironment[
  fileext=lop,
  listname={List of Protocols},
  name=Protocol,
  placement=tbp
]{protocol}

\definecolor{PrismBlue}{HTML}{4477AA}
\definecolor{PrismCyan}{HTML}{66CCEE}
\definecolor{PrismGreen}{HTML}{228833}
\definecolor{PrismYellow}{HTML}{CCBB44}
\definecolor{PrismRed}{HTML}{EE6677}
\definecolor{PrismPurple}{HTML}{AA3377}
\definecolor{PrismOrange}{HTML}{E69F00}
\definecolor{PrismGrey}{HTML}{777777}
\definecolor{PrismLight}{HTML}{F5F8FA}

\usetikzlibrary{arrows.meta,positioning,calc,fit,matrix,shapes.geometric,decorations.pathreplacing}
\usepgfplotslibrary{groupplots,fillbetween}
\pgfplotsset{compat=1.18}

\pgfplotsset{
  prism axis/.style={
    axis line style={black!60},
    tick style={black!60},
    label style={font=\small},
    tick label style={font=\scriptsize},
    legend style={font=\scriptsize,draw=none,fill=none},
    grid=major,
    grid style={black!8},
    scaled ticks=false,
    every axis plot/.append style={line width=1pt}
  }
}

\newcolumntype{L}[1]{>{\raggedright\arraybackslash}p{#1}}
\newcolumntype{Y}{>{\raggedright\arraybackslash}X}

\usepackage[most]{tcolorbox}
\newtcolorbox{takeaway}{
  enhanced,colback=PrismBlue!4,colframe=PrismBlue!65!black,
  boxrule=0.55pt,arc=1.5pt,left=5pt,right=5pt,top=2pt,bottom=2pt,
  fonttitle=\bfseries,fontupper=\small,title={Takeaway},before skip=3pt,after skip=1pt
}

\setlist{nosep,leftmargin=1.45em}
\usepackage{placeins}
\usepackage{xurl,hyperref}
\hypersetup{
  colorlinks=true,
  linkcolor=PrismBlue!80!black,
  citecolor=PrismGreen!60!black,
  urlcolor=PrismBlue!80!black,
  bookmarksnumbered=true,
  pdftitle={Permutation-based Reasoning and Intelligence Search Method (PRISM)},
  pdfauthor={Blessings Mambwe}
}
\PassOptionsToPackage{capitalize,nameinlink,noabbrev}{cleveref}
\PassOptionsToPackage{backend=biber,style=numeric-comp,sorting=none,maxbibnames=99}{biblatex}
\usepackage{cleveref,biblatex}
\crefname{protocol}{Protocol}{Protocols}
\Crefname{protocol}{Protocol}{Protocols}
\crefname{algorithm}{Algorithm}{Algorithms}
\Crefname{algorithm}{Algorithm}{Algorithms}

\DeclareMathOperator*{\argmax}{arg\,max}
\newcommand{\Sn}{\mathbb{S}_n}
\newcommand{\rhoone}{\rho_1}
\newcommand{\fdc}{\mathrm{FDC}}
\newcommand{\fit}{F}
\newcommand{\archiveurl}{\url{https://github.com/bleymambwe/PRISM}}

%% file: sections/00_abstract.tex
\begin{abstract}
PRISM addresses permutation-valued decisions in systems whose components are fixed but whose
order is not: compiler passes,
preprocessing steps, neural modules, or reasoning instructions. Before selecting an optimizer, it
asks whether structured search is likely to beat uniform sampling. PRISM's \emph{pre-flight} uses a
small fixed measurement budget to estimate one-move fitness
autocorrelation and fitness--distance correlation, forecasts the search regime, selects a method,
and records the forecast beside the outcome. On exactly enumerated landscapes, the protocol
predicts its own failures: on a 5,040-ordering neural landscape, evolutionary search finds an
optimum in 19 of 40 runs, versus 30 of 40 for uniform sampling, as near-zero local correlation
forecasts. In the flagship application, all 720 orderings of six fixed reasoning instructions span
6.3\% to 96.9\% accuracy on a GSM8K subset. The position effects transfer across model families
and survive automated rewriting of the instruction text. Cross-domain evidence includes six
scientific-pipeline systems, 18 registered tabular architecture forecasts, and two preliminary
reinforcement-learning ordering surfaces. Code and artifact index: \archiveurl.
\end{abstract}

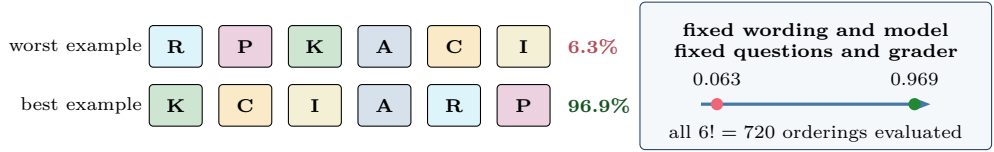
\begin{figure}[H]
\centering
\begin{tikzpicture}[
  x=0.92cm,y=0.78cm,
  cell/.style={draw,rounded corners=1.5pt,minimum width=0.70cm,minimum height=0.55cm,
    font=\scriptsize\bfseries},
  note/.style={font=\scriptsize,align=left},
  conditions/.style={draw=PrismBlue!55!black,fill=PrismLight,rounded corners=2pt,
    minimum width=4.65cm,minimum height=1.95cm}
]
\foreach \x/\v/\c in {0/R/PrismCyan,1/P/PrismPurple,2/K/PrismGreen,3/A/PrismBlue,4/C/PrismOrange,5/I/PrismYellow}
  \node[cell,fill=\c!22] at (\x,1.0) {\v};
\node[note,anchor=east] at (-0.35,1.0) {worst example};
\node[note,anchor=west,font=\scriptsize\bfseries,text=PrismRed!75!black] at (5.5,1.0) {6.3\%};

\foreach \x/\v/\c in {0/K/PrismGreen,1/C/PrismOrange,2/I/PrismYellow,3/A/PrismBlue,4/R/PrismCyan,5/P/PrismPurple}
  \node[cell,fill=\c!22] at (\x,0.0) {\v};
\node[note,anchor=east] at (-0.35,0.0) {best example};
\node[note,anchor=west,font=\scriptsize\bfseries,text=PrismGreen!60!black] at (5.5,0.0) {96.9\%};

\node[conditions] at (9.2,0.50) {};
\node[font=\scriptsize\bfseries,align=center,text width=4.15cm] at (9.2,1.08)
  {fixed wording and model\\fixed questions and grader};
\draw[-{Stealth[length=5pt]},very thick,PrismBlue] (7.55,0.02)--(10.85,0.02);
\fill[PrismRed] (7.78,0.02) circle (2.3pt);
\fill[PrismGreen] (10.62,0.02) circle (2.3pt);
\node[font=\scriptsize,anchor=south] at (7.78,0.18) {0.063};
\node[font=\scriptsize,anchor=south] at (10.62,0.18) {0.969};
\node[font=\scriptsize,anchor=north] at (9.2,-0.18)
  {all \(6! = 720\) orderings evaluated};
\end{tikzpicture}
\caption{\textbf{Order alone moves measured accuracy by 90.6 percentage points.} The strips show
one worst and one best ordering of RESTATE (R), IDENTIFY (I), PLAN (P), COMPUTE (C), CHECK (K),
and ANSWER (A). Every ordering is evaluated on the same 32 GSM8K questions with fixed wording,
deterministic decoding, and the same grader.}
\label{fig:hook}
\end{figure}

%% file: sections/01_introduction.tex
\section{Introduction}
\label{sec:introduction}

Many systems have already selected the right components but have not justified their order. A
compiler has passes, a scientific pipeline has preprocessing operations, a neural model has
modules, and a prompt has instructions. Once the component multiset is fixed, a candidate is a
permutation \(\pi\in\Sn\), and the optimization problem is
\begin{equation}
  \pi^*\in\argmax_{\pi\in\Sn}\fit(\pi),
  \label{eq:objective}
\end{equation}
where \(\fit\) is the task score. Composition is generally non-commutative: moving a nonlinear
operation, a normalizer, or an instruction changes what later components receive.
The magnitude of this effect in the exhaustive instruction study is previewed in
\cref{fig:hook}.

Fixing the multiset is a genuine restriction on a broader configuration space. If \(n\) ordered
slots draw from \(k\) component types, unrestricted assignment yields \(k^n\) sequences. With
fixed counts \(c_1,\ldots,c_k\), \(\sum_i c_i=n\), the reachable arrangements number
\begin{equation}
  |S(\mathbf c)|=\frac{n!}{c_1!\cdots c_k!}\le k^n,
  \label{eq:multiset-space}
\end{equation}
because they form one subset of the length-\(n\) sequences over the same alphabet. When all
components are distinct, as in the experiments here, \(|S|=n!\). This restriction buys exact
ground truth at small \(n\): six instructions have 720 orderings, compared with \(6^6=46{,}656\)
unrestricted six-slot configurations.

Restriction is not simplification. Cardinality does not determine search difficulty. The exact
parity landscape over \(7!=5{,}040\) orderings defeats elitist evolutionary search, which finds an
optimum in 19 of 40 runs against 30 of 40 for uniform sampling. A matched-operator synthetic
landscape over \(16!\approx2.09\times10^{13}\) orderings is solved in roughly 405 generations. A
five-thousand-point space can be harder than a twenty-trillion-point one. The difference is
structure: whether a search move preserves fitness information and whether better candidates tend
to lie closer to an optimum. Neither property is visible in \(n\), so both must be measured.

In this work, we introduce a pre-flight protocol that spends a small, fixed budget before the main
search. It measures local fitness preservation and global guidance, checks whether the evaluator
can distinguish candidates, forecasts whether structured search should help, and selects an
\emph{executor}---the method that spends the remaining budget. Possible executors include exact
enumeration, uniform sampling, an evolutionary searcher, a surrogate, or a transfer-guided list.
The forecast, rather than any single executor, is the contribution.

Our contributions are four results:
\begin{enumerate}[label=\textbf{C\arabic*.}]
  \item A two-statistic pre-flight separates operator-aligned, near-random, and deceptive regimes
  and turns null results into testable forecasts (\cref{sec:protocol,sec:landscapes}).
  \item Exact enumeration exposes a decisive boundary: structured search can be worse than
  uniform sampling when its neighbourhood contains almost no fitness signal
  (\cref{sec:falsification}).
  \item A complete 720-order map of fixed reasoning instructions reveals a 6.3\%--96.9\% accuracy
  range, transferable position effects, and ordering variation that survives automated wording
  optimization (\cref{sec:llm-landscape,sec:transfer-baselines}).
  \item Cross-domain studies cover six scientific-computing pipelines, registered tabular
  architecture forecasts, and preliminary option-ordering and curriculum-ordering surfaces in
  reinforcement learning (\cref{sec:cross-domain,sec:applications}).
\end{enumerate}

The argument can be read as five claims: (1) order is a real optimization coordinate; (2) space
size is not a difficulty measure; (3) search operators define the relevant locality; (4) exact
ground truth makes failures visible; and (5) forecasts must be retained beside outcomes, including
misses. The rest of the paper follows that order, then closes with cross-domain calibration and
specific limitations.
The final synthesis in \cref{tab:summary-v3} connects each headline finding to its practical
recommendation.

\begin{takeaway}
Ordering is a real optimization surface, but the size of the space says nothing about whether
searching it will pay. Difficulty must be measured from landscape structure before budget is spent.
\end{takeaway}

%% file: sections/01_terminology.tex
\section{Terminology and Prerequisites}
\label{sec:terminology}

PRISM assumes a constrained permutation problem rather than unrestricted system design. A
\emph{component} is a fixed item whose internal content does not change during the study. A
\emph{candidate} is one legal ordering \(\pi\) of those components, and the \emph{evaluator}
assigns its task score \(\fit(\pi)\). The resulting \emph{fitness landscape} is the set of legal
orderings together with their scores. A move operator---swap, insertion, inversion, or
scramble---defines a neighbourhood on that landscape. An \emph{executor} is the method selected to
spend the post-diagnostic evaluation budget, such as enumeration, sampling, a surrogate, or the
evolutionary searcher.
These objects and their dependencies are summarized in \cref{fig:terminology-v2}.

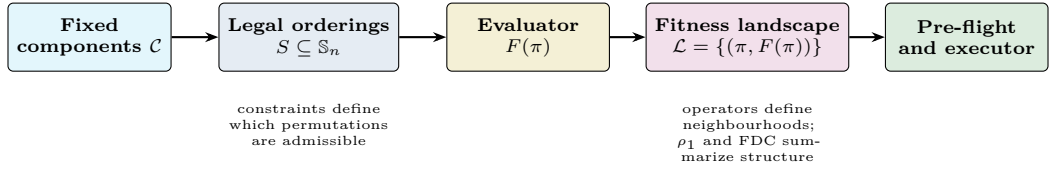
\begin{figure}[tbp]
\centering
\begin{tikzpicture}[
  box/.style={draw,rounded corners=2pt,minimum width=2.15cm,minimum height=0.9cm,
    align=center,font=\scriptsize\bfseries},
  arr/.style={-{Stealth[length=5pt]},thick},
  note/.style={font=\tiny,align=center,text width=2.4cm}
]
\node[box,fill=PrismCyan!18] (components) at (0,0) {Fixed\\components \(\mathcal C\)};
\node[box,fill=PrismBlue!14] (space) at (2.9,0) {Legal orderings\\\(S\subseteq\Sn\)};
\node[box,fill=PrismYellow!22] (eval) at (5.8,0) {Evaluator\\\(\fit(\pi)\)};
\node[box,fill=PrismPurple!15] (landscape) at (8.7,0) {Fitness landscape\\\(\mathcal L=\{(\pi,\fit(\pi))\}\)};
\node[box,fill=PrismGreen!16] (decision) at (11.6,0) {Pre-flight\\and executor};
\draw[arr] (components)--(space);
\draw[arr] (space)--(eval);
\draw[arr] (eval)--(landscape);
\draw[arr] (landscape)--(decision);
\node[note,below=0.30cm of space] {constraints define which permutations are admissible};
\node[note,below=0.30cm of landscape] {operators define neighbourhoods; \(\rhoone\) and \(\fdc\) summarize structure};
\end{tikzpicture}
\caption{\textbf{Objects required for a PRISM study.} The landscape is not merely the search-space
size: it is the scored structure induced by the evaluator and the chosen neighbourhood operators.
The empirical landscape regimes are visualized in \cref{fig:regime-map-v2}.}
\label{fig:terminology-v2}
\end{figure}

Five prerequisites should hold before optimization begins: (1) the component set and legal-order
constraints are explicit; (2) candidate scores are reproducible or accompanied by a declared noise
model; (3) the evaluator has sufficient resolution to distinguish orderings; (4) repeated
evaluations can be identified so distinct-evaluation budgets are comparable; and (5) at least one
uninformed and one domain-specific baseline are available. If these conditions fail, the correct
action is to repair the experimental design rather than interpret an optimizer ranking.

\begin{takeaway}
A permutation experiment is scientifically meaningful only when the components, legal space,
evaluator, landscape neighbourhood, and comparison budget are defined before search.
\end{takeaway}

%% file: sections/02_protocol.tex
\section{The Pre-Flight Protocol}
\label{sec:protocol}

A pre-flight uses a small probe set \(P\subset\Sn\) before the main optimization budget. Repeated
orderings are served from cache and do not count again; we call this the \emph{distinct-evaluation
budget}. The probe first checks score variance and the density of top ties. If the evaluator cannot
distinguish candidates, optimizer comparison is premature.

For a one-move operator \(o\), local fitness preservation is estimated by
\begin{equation}
  \rhoone(o)=\operatorname{corr}\!\left(\fit(\pi),\fit(o(\pi))\right),
  \label{eq:rho1}
\end{equation}
over sampled parent--child pairs. A positive value means that the neighbourhood preserves fitness
information; a value near zero means that one move resembles a fresh draw. Global guidance is
measured by fitness--distance correlation
\begin{equation}
  \fdc=\operatorname{corr}\!\left(\fit(\pi),d(\pi,\Omega)\right),
  \label{eq:fdc}
\end{equation}
where \(\Omega\) is the exact optimum set when known and otherwise the best frozen probe sample.
Because fitness is maximized, negative \(\fdc\) indicates that better candidates tend to lie nearer
the target; near-zero values imply weak global guidance; strongly positive values warn that local
improvement points away from the optimum. The protocol below routes these measurements to the
executor that spends the remaining budget, while quantitative decision bands appear in
\cref{tab:decision-rule-v2,fig:regime-map-v2}.

\begin{protocol}[tbp]
\caption{PRISM pre-flight protocol}
\label{proto:preflight}
\small
\textbf{Input:} frozen evaluator \(\fit\), component multiset, probe budget \(B_p\), main budget
\(B\), candidate executors.\\[-0.25em]
\textbf{Output:} recorded forecast, selected executor, verified outcome.
\begin{enumerate}
  \item Sample \(B_p\) distinct orderings; measure variance, tie density, and evaluator resolution.
  \item Estimate \(\rhoone(o)\) for each meaningful move \(o\); estimate exact or proxy \(\fdc\).
  \item Forecast the regime and select enumeration, sampling, evolution, a surrogate, or transfer.
  \item Execute exactly \(B\) distinct evaluations; repeated proposals do not increase the count.
  \item Record forecast, outcome, uncertainty, and any mismatch in the forecast ledger.
\end{enumerate}
\end{protocol}

\begin{table}[tbp]
\centering
\small
\caption{Conservative decision rule. Thresholds are empirical operating bands, not universal
constants.}
\label{tab:decision-rule-v2}
\begin{tabularx}{0.94\linewidth}{L{4.3cm}Y}
\toprule
\textbf{Pre-flight observation} & \textbf{Action} \\
\midrule
Negligible variance or dense top ties & Improve the evaluator or sample modestly and stop. \\
One operator has clearly largest \(\rhoone\) & Use its neighbourhood if search is otherwise justified. \\
Materially negative \(\fdc\) & Structured search may exploit global guidance. \\
Near-zero or modestly positive \(\fdc\) & Use uniform sampling as the method of record. \\
Strongly positive \(\fdc\) & Avoid local evolutionary search. \\
\bottomrule
\end{tabularx}
\end{table}

\begin{algorithm}[tbp]
\caption{Elitist permutation executor}
\label{alg:executor}
\small
\KwIn{fitness \(\fit\), move \(o\), population size \(m\), distinct-evaluation budget \(B\)}
\KwOut{best evaluated ordering}
Initialize \(m\) random permutations and evaluate unseen candidates\;
\While{fewer than \(B\) distinct orderings have been evaluated}{
  select a parent by tournament; propose \(\pi'=o(\pi)\)\;
  evaluate \(\pi'\) only if its cache key is new\;
  replace the worst non-elite member; retain the best-so-far ordering\;
}
\Return best cached ordering\;
\end{algorithm}

\begin{takeaway}
Two cheap measurements separate locally informative, near-random, and deceptive regimes. The
protocol can therefore select uniform sampling as a successful outcome rather than treating it as
a failed optimizer.
\end{takeaway}

%% file: sections/02_related_work.tex
\section{Related Work}
\label{sec:related-work}

\paragraph{Landscape analysis and permutation search.}
One-step autocorrelation and fitness--distance correlation are classical measures of local
preservation and global guidance~\cite{weinberger1990correlated,jones1995fdc}. Permutation
landscape analysis further connects distance metrics and move operators to position, precedence,
and adjacency structure~\cite{cicirello2022landscape}; convergence under irreducible mutation and
elitist retention is standard~\cite{rudolph1994convergence}. PRISM does not claim these ingredients
as new. Its contribution is to operationalize them as a low-budget, recorded selector among
enumeration, sampling, evolutionary search, and a structure-matched surrogate, then test the
forecast on exact and live application landscapes.

\paragraph{Prompt order, position, and calibration.}
Lu et al.\ show that reordering few-shot demonstrations can produce large performance differences
and that effective orders need not transfer across models~\cite{lu2022fantastically}. Zhao et al.
reduce prompt-induced instability through contextual calibration~\cite{zhao2021calibrate}. These
studies motivate robustness interventions; PRISM instead asks whether an ordering surface has
enough measurable structure to justify search. Lost-in-the-middle effects concern the placement of
retrieved evidence in long contexts~\cite{liu2024lostmiddle}, while Sclar et al.\ quantify
sensitivity to formatting choices~\cite{sclar2024sensitivity}. The present experiments hold
component wording and rendering fixed and vary only the order of instruction modules.

Option-order studies expose related positional bias in multiple-choice evaluation. Reordering
answer options can substantially change measured performance~\cite{pezeshkpour2024optionorder},
and selection bias toward option identifiers motivates explicit debiasing
methods~\cite{zheng2024mcq}. Those works permute answer choices; PRISM permutes the internal order
of a fixed reasoning scaffold and characterizes the full six-component landscape. Recent causal
attention and presentation-order accounts provide candidate mechanisms for why later information
may be used differently~\cite{ok2026lost,zhou2026curse}, but they do not determine whether a given
ordering landscape is searchable.

\paragraph{Prompt optimization.}
Chain-of-thought and self-consistency alter elicited computation
~\cite{wei2022cot,wang2022selfconsistency}; APE, OPRO, and DSPy optimize prompt text or programs
~\cite{zhou2022ape,yang2023opro,khattab2023dspy}. PRISM freezes component content and isolates order
as a separate coordinate. The OPRO experiment therefore tests complementarity rather than claiming
that ordering replaces content optimization.

%% file: sections/03_landscapes.tex
\section{Landscapes, Operators, and Locality}
\label{sec:landscapes}

Search operators define what ``nearby'' means. Absolute-position landscapes reward elements in
specific slots and align with swap moves. Precedence landscapes reward pairwise order and align
with insertion. Adjacency landscapes reward neighbouring pairs and align with inversion. Scramble
is a broad disruption rather than a precise neighbourhood. These types are diagnostic ideals, not
a claim that real applications are pure.
The three operator-landscape pairings are illustrated in \cref{fig:operators-v2}.

\begin{figure}[tbp]
\centering
\includegraphics[width=0.92\linewidth]{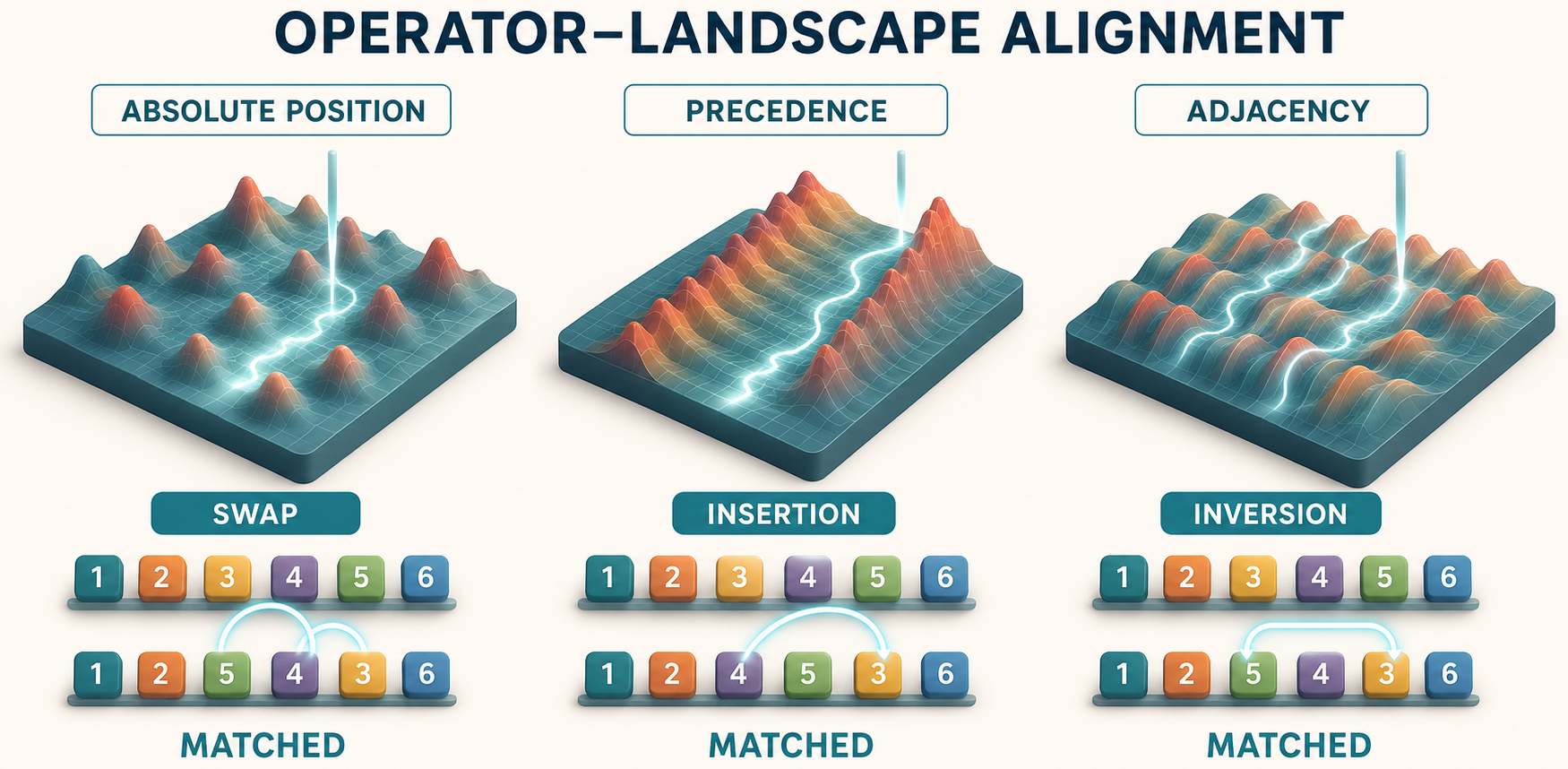}
\caption{\textbf{Operator-to-landscape alignment.} The conceptual illustration pairs absolute
position with swap, precedence with insertion, and adjacency with inversion. It is a schematic,
not a measured surface; the observed censoring and locality statistics are reported in the text and
\cref{fig:regime-map-v2}.}
\label{fig:operators-v2}
\end{figure}

The difference is operationally large. At \(n=16\), the matched operator solves the
absolute-position objective in 405 generations, the precedence objective in 396, and the adjacency
objective in 1,199, while mismatched operators often reach the 10,000-generation cap without an
optimum. We call such a run \emph{censored}: it exhausts its cap without success, and it remains in
the denominator rather than being discarded. When the type is unknown, a uniform four-operator
portfolio removes censoring on these typed landscapes at an observed \(1.1\)--\(4.1\times\) cost.

The pre-flight automates the choice. With 100 move pairs per operator, the pair-sampled
\(\rhoone\) selector identifies the matched move with mean accuracy 0.925 across typed synthetic
and enumerated language-model landscapes. Representation can matter even more than population
tuning: on a 5,040-order precedence landscape, a ridge surrogate using pairwise precedence
features reaches the unique optimum in 14.8 evaluations, bootstrap 95\% interval [14,15], versus
54.2 [50,59] for matched elitist search.

\begin{figure}[tbp]
\centering
\begin{tikzpicture}
\begin{axis}[
  prism axis,width=0.78\linewidth,height=0.52\linewidth,
  xlabel={Selected-operator one-move autocorrelation \(\rhoone\)},
  ylabel={Fitness--distance correlation \(\fdc\)},
  xmin=-0.02,xmax=0.82,ymin=-0.90,ymax=0.90,
  xtick={0,0.2,0.4,0.6,0.8},ytick={-0.8,-0.4,0,0.4,0.8},
  xticklabel style={/pgf/number format/fixed,/pgf/number format/precision=1},
  yticklabel style={/pgf/number format/fixed,/pgf/number format/precision=1},
  legend style={at={(0.02,0.02)},anchor=south west,fill=white,fill opacity=0.9,text opacity=1},
  axis on top
]
\fill[PrismGreen!9] (axis cs:0.30,-0.90) rectangle (axis cs:0.82,-0.05);
\fill[PrismYellow!13] (axis cs:-0.02,-0.05) rectangle (axis cs:0.82,0.30);
\fill[PrismRed!9] (axis cs:-0.02,0.30) rectangle (axis cs:0.82,0.90);
\node[font=\scriptsize\itshape,text=PrismGreen!55!black] at (axis cs:0.63,-0.72) {guiding};
\node[font=\scriptsize\itshape,text=PrismYellow!45!black] at (axis cs:0.62,0.14) {sample / caution};
\node[font=\scriptsize\itshape,text=PrismRed!65!black] at (axis cs:0.63,0.70) {deceptive};

\addplot+[only marks,mark=triangle*,mark size=2.8pt,color=black] coordinates {
  (0.049,-0.252) (0.058,-0.248) (0.044,-0.202) (0.021,-0.063)};
\addlegendentry{neural}
\addplot+[only marks,mark=square*,mark size=2.8pt,color=PrismBlue] coordinates {
  (0.655,-0.829) (0.777,-0.321) (0.675,-0.254)};
\addlegendentry{typed, matched}
\addplot+[only marks,mark=square*,mark size=3pt,color=PrismRed!80!black]
  coordinates {(0.649,0.783)};
\addlegendentry{deceptive}
\addplot+[only marks,mark=diamond*,mark size=3.3pt,color=PrismPurple] coordinates {
  (0.547,-0.346) (0.750,-0.095) (0.178,0.197)};
\addlegendentry{LLM / math}
\node[font=\tiny,anchor=north west] at (axis cs:0.025,-0.063) {parity};
\node[font=\tiny,anchor=south west] at (axis cs:0.649,0.783) {deceptive};
\node[font=\tiny,anchor=south east] at (axis cs:0.547,-0.346) {six instructions};
\node[font=\tiny,anchor=north east] at (axis cs:0.750,-0.095) {eight modules};
\node[font=\tiny,anchor=south west] at (axis cs:0.178,0.197) {MATH-500};
\end{axis}
\end{tikzpicture}
\caption{\textbf{The two pre-flight axes separate guiding, near-random, and deceptive regimes.}
The four neural, three typed synthetic, and one deceptive point use exact FDC and the largest
\(\rhoone\) among swap, insertion, and inversion; scramble is excluded because identity moves can
inflate its correlation. Three later language-model and mathematics probes use the same axes with
frozen-sample approximations where exact enumeration is unavailable. The shaded regions summarize
the conservative rule in \cref{tab:decision-rule-v2}; optimum density and evaluator resolution
remain additional gates.}
\label{fig:regime-map-v2}
\end{figure}
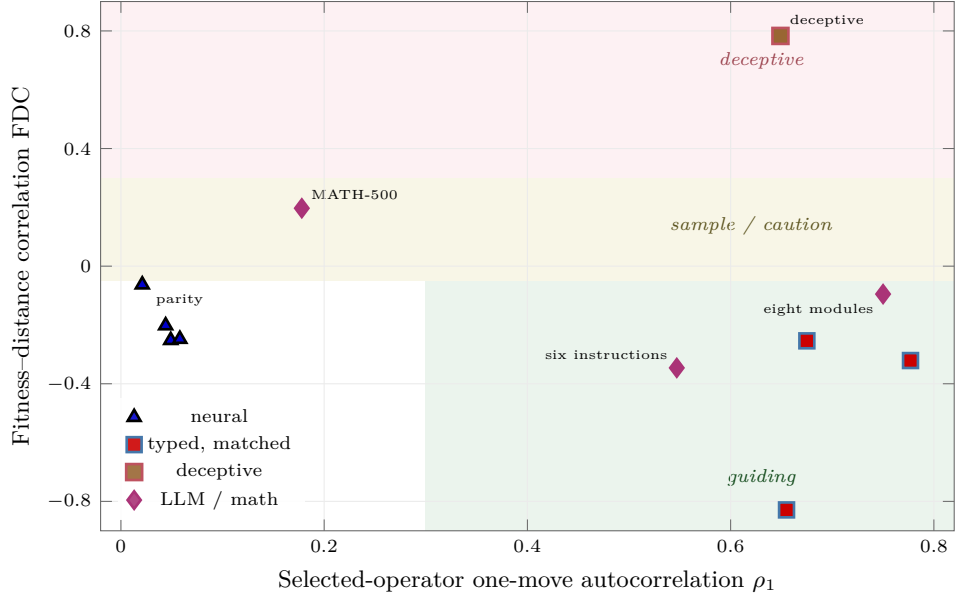

\begin{takeaway}
Search operators define locality. A mismatched move can prevent success rather than merely slow it,
while a representation aligned with the measured structure can outperform population search.
\end{takeaway}

%% file: sections/04_falsification.tex
\section{Ground Truth and the Falsification}
\label{sec:falsification}

Search evaluation is circular when the optimizer's best observation is also treated as the answer
key. Repeated stochastic queries make the maximum optimistically biased, and a method that explores
more candidates can redefine the apparent target in its own favour. We therefore enumerate every
ordering whenever feasible, freeze sampled pools otherwise, and compare methods at equal
distinct-evaluation budgets. Exact landscapes report the optimum set, its density, hit probability,
and regret rather than distance to a method-dependent best observation.
The complete parity counterexample and its audited hit counts are summarized in
\cref{fig:falsification-v2}.

\begin{figure}[tbp]
\centering
\includegraphics[width=0.72\linewidth,height=0.29\textheight,keepaspectratio]
  {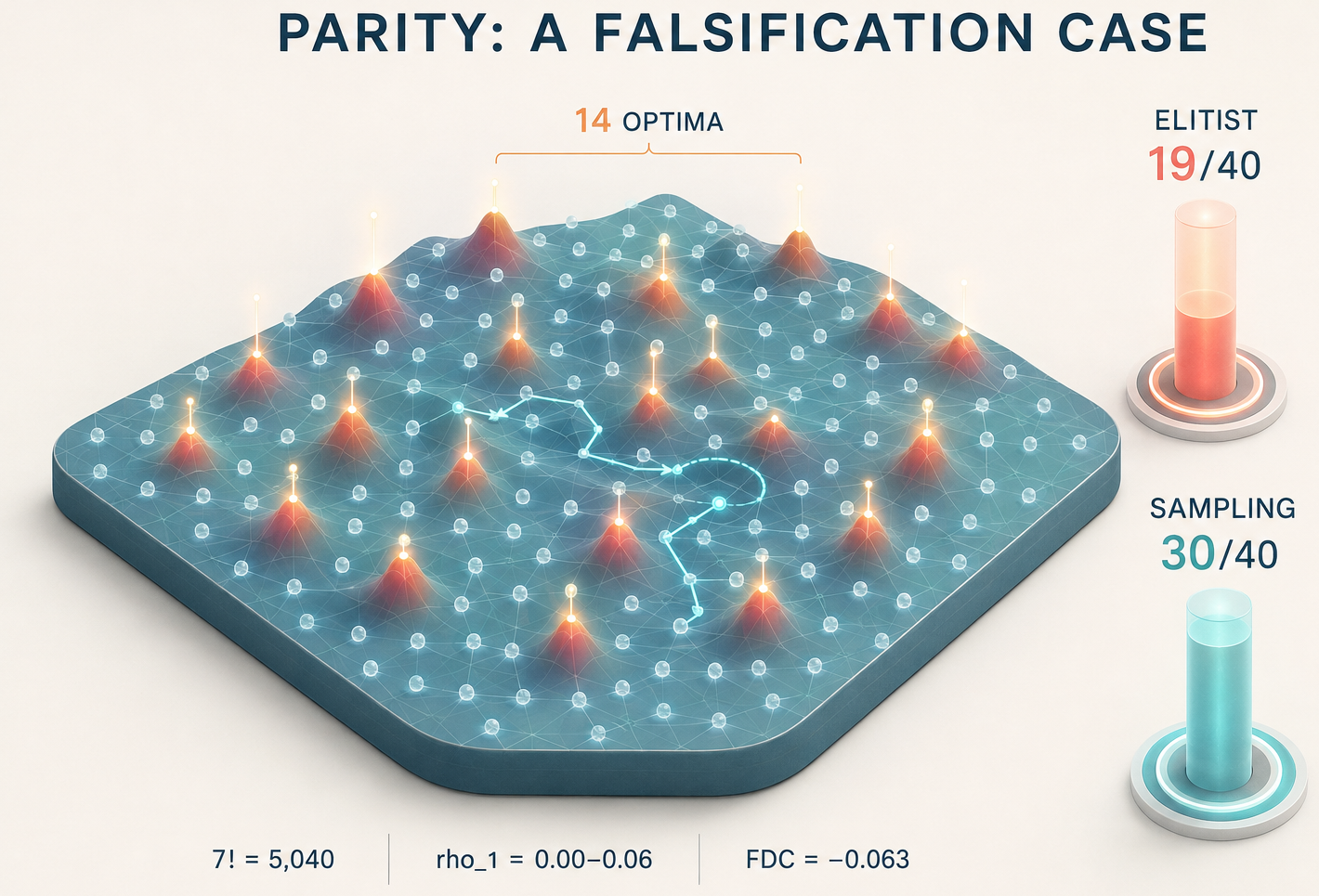}
\caption{\textbf{Parity is the central falsification case.} The schematic marks 14 optima among
5,040 orderings and audited hit counts of 19/40 for elitist search versus 30/40 for sampling.
Wilson score 95\% intervals are [0.33,0.63] and [0.60,0.86].}
\label{fig:falsification-v2}
\end{figure}

The decisive test is a deterministic neural parity landscape with seven fixed modules. It contains
14 optima among \(7!=5{,}040\) orderings. The study was intended to show that scale-aware tuning
would preserve an evolutionary advantage. It did not. Default elitist search found an optimum in
4 of 15 exploratory runs. Larger populations and stronger mutation removed population collapse but
required 384 mean evaluations, versus 318 for uniform sampling.

The failure is structural rather than a tuning anecdote. Swap, insertion, and inversion have
\(\rhoone\) between 0.00 and 0.06, and exact \(\fdc=-0.063\). One local move therefore carries
almost no information about the next candidate, and global distance provides little additional
guidance. Aging and surrogate variants raise coverage, but no tested structured method establishes
a decisive advantage over the uninformed baseline. The pre-flight predicts this boundary before
the search comparison is run.

This negative result changed the research claim. The supported statement is not ``evolution beats
random''; it is that the benefit of structured permutation search is conditional on measurable
landscape structure. A correct forecast that sampling should be competitive saves budget and is a
success for the protocol.

\begin{takeaway}
Structured search is not universally better than random sampling. On a landscape with almost no
local signal, the honest baseline wins, and the pre-flight identifies that condition in advance.
\end{takeaway}

%% file: sections/05_llm_landscape.tex
\section{The Instruction-Ordering Landscape}
\label{sec:llm-landscape}

The flagship landscape freezes six one-sentence modules---RESTATE, IDENTIFY, PLAN, COMPUTE,
CHECK, and ANSWER---and changes only their order. Each permutation is rendered as numbered steps
before the same 32-question GSM8K subset~\cite{cobbe2021gsm8k}. Gemini 2.5 Flash-Lite is queried
with deterministic decoding; answers are scored by a frozen numeric extractor; every
ordering--question response is cached. All \(6!=720\) orderings are evaluated.

Accuracy ranges from 0.063 to 0.969, with mean 0.717 and standard deviation 0.230. The effect is
not merely a few exceptional prompts: marginal position averages in
\cref{fig:position-heatmap-v2} reveal a coherent mechanism.

\begin{figure}[tbp]
\centering
\begin{tikzpicture}[x=1.05cm,y=0.64cm]
\foreach \x in {1,...,6}{\node[font=\scriptsize] at (\x,-0.65) {\x};}
\node[font=\scriptsize] at (3.5,-1.05) {prompt position};
\foreach \y/\name in {5/RESTATE,4/IDENTIFY,3/PLAN,2/COMPUTE,1/CHECK,0/ANSWER}
  \node[font=\scriptsize,anchor=east] at (0.42,\y) {\name};

\foreach \x/\v/\p in {1/0.70/38,2/0.72/41,3/0.72/41,4/0.73/43,5/0.73/43,6/0.70/38}
  \node[draw=white,fill=PrismBlue!\p!white,minimum width=1.05cm,minimum height=0.64cm,
    inner sep=0pt,font=\scriptsize] at (\x,5) {\v};
\foreach \x/\v/\p in {1/0.76/47,2/0.74/44,3/0.72/41,4/0.71/40,5/0.71/40,6/0.65/31}
  \node[draw=white,fill=PrismBlue!\p!white,minimum width=1.05cm,minimum height=0.64cm,
    inner sep=0pt,font=\scriptsize] at (\x,4) {\v};
\foreach \x/\v/\p in {1/0.77/49,2/0.72/41,3/0.70/38,4/0.71/40,5/0.72/41,6/0.68/35}
  \node[draw=white,fill=PrismBlue!\p!white,minimum width=1.05cm,minimum height=0.64cm,
    inner sep=0pt,font=\scriptsize] at (\x,3) {\v};
\foreach \x/\v/\p in {1/0.91/60,2/0.78/50,3/0.71/40,4/0.64/29,5/0.58/20,6/0.68/35}
  \node[draw=white,fill=PrismBlue!\p!white,minimum width=1.05cm,minimum height=0.64cm,
    inner sep=0pt,font=\scriptsize] at (\x,2) {\v};
\foreach \x/\v/\p in {1/0.73/43,2/0.76/47,3/0.71/40,4/0.69/37,5/0.68/35,6/0.72/41}
  \node[draw=white,fill=PrismBlue!\p!white,minimum width=1.05cm,minimum height=0.64cm,
    inner sep=0pt,font=\scriptsize] at (\x,1) {\v};
\foreach \x/\v/\p in {1/0.43/5,2/0.58/20,3/0.73/43,4/0.81/54,5/0.87/59,6/0.87/59}
  \node[draw=white,fill=PrismBlue!\p!white,minimum width=1.05cm,minimum height=0.64cm,
    inner sep=0pt,font=\scriptsize] at (\x,0) {\v};
\end{tikzpicture}
\caption{\textbf{ANSWER is favoured late and COMPUTE early.} Each cell is mean accuracy over all
720 orderings in which the named module occupies the given position. ANSWER rises from 0.435 in
the first position to 0.870 in the last; COMPUTE falls from 0.908 first to 0.585 fifth.}
\label{fig:position-heatmap-v2}
\end{figure}

Every one of the three worst orderings places ANSWER before COMPUTE; the best examples reverse
that precedence. This is consistent with the idea that presentation order should respect the order
in which information is used~\cite{zhou2026curse}, while not identifying a unique causal
mechanism. The pre-flight selects insertion: \(\rhoone=0.547\), compared with 0.411 for swap and
0.297 for inversion. Exact \(\fdc=-0.346\) confirms guiding structure.

Guidance does not guarantee a large optimizer advantage. Sixty orderings, 8.3\% of the space,
attain the maximum measured accuracy. Across 40 replay runs, evolutionary search reaches an
optimum in 9.9 mean distinct evaluations, bootstrap percentile 95\% interval [7,13], versus 10.8
[8,14] for uniform sampling. The intervals overlap. An earlier 15-run estimate suggested a large
speedup; the higher-powered audit withdrew that claim.

Two additional modules expand the space to \(8!=40{,}320\). A probe of 222 unbiased orderings on
20 questions finds accuracy from 0.10 to 1.00 and selects insertion at \(\rhoone=0.75\), but
approximate \(\fdc=-0.095\) forecasts that sampling will be competitive. Evolutionary search and
uniform sampling both reach measured accuracy 1.0 at budgets 30, 60, and 120 in all four replay
runs. At this resolution, one question changes accuracy by 0.05 and perfect ties are dense; the
protocol recommends a modest sample or a better evaluator rather than a more elaborate searcher.

\begin{takeaway}
A fixed set of six reasoning instructions spans 90.6 accuracy points under order alone, with
interpretable position effects. Dense optima nevertheless make uniform sampling competitive.
\end{takeaway}

%% file: sections/06_transfer_baselines.tex
\section{Transfer and Baselines}
\label{sec:transfer-baselines}

The exact six-module landscape supports transfer at three resolutions. First, a score constructed
from its module--position table correlates with fitness on the 222 unbiased eight-module orderings
at Spearman \(\rho=0.665\), bootstrap 95\% interval [0.582,0.733]. Second, source-selected top
orderings are tested on fixed pools for Gemma-3-27B, Qwen3-30B-A3B, and Llama-3.2-3B. Fine rank
transfer varies, but the practical top-ordering advantage is positive on all three targets, with
bootstrap intervals excluding zero. Third, on a filtered MATH-500 hard-reasoning
pool~\cite{hendrycks2021math}, the easy-task position score correlates with target fitness at
\(r=0.432\), while the pre-flight's random-competitive forecast is confirmed by overlapping search
intervals.
\Cref{fig:transfer-baselines-v2} reports both the cross-family effects and the persistence of
ordering variation after wording optimization.

\begin{figure}[tbp]
\centering
\begin{subfigure}[t]{0.58\linewidth}
\centering
\begin{tikzpicture}
\begin{axis}[
  prism axis,width=0.84\linewidth,height=0.58\linewidth,
  xmin=-0.03,xmax=0.26,xtick={0,0.05,0.10,0.15,0.20,0.25},
  xticklabel style={/pgf/number format/fixed,/pgf/number format/precision=2},
  xlabel={Top-ordering advantage over random},
  symbolic y coords={Qwen3-30B-A3B,Llama-3.2-3B,Gemma-3-27B},
  ytick=data,y dir=reverse,xmajorgrids=true,ymajorgrids=false,
  yticklabel style={font=\scriptsize}
]
\addplot+[only marks,mark=*,mark size=2.6pt,color=PrismBlue,
  error bars/.cd,x dir=both,x explicit] coordinates {
  (0.032,Qwen3-30B-A3B) +- (0.019,0)
  (0.061,Llama-3.2-3B) +- (0.058,0)
  (0.123,Gemma-3-27B) +- (0.113,0)};
\draw[PrismRed,dashed,thick] (axis cs:0,Qwen3-30B-A3B)--(axis cs:0,Gemma-3-27B);
\end{axis}
\end{tikzpicture}
\caption{Cross-family transfer with 95\% intervals.}
\end{subfigure}\hfill
\begin{subfigure}[t]{0.39\linewidth}
\centering
\begin{tikzpicture}
\begin{axis}[
  prism axis,width=0.92\linewidth,height=0.86\linewidth,
  ybar,bar width=18pt,ymin=0,ymax=0.082,
  ylabel={Ordering-driven std. dev.},
  symbolic x coords={Original,Rewritten},xtick=data,
  x tick label style={rotate=15,anchor=east},
  ytick={0,0.02,0.04,0.06,0.08},
  yticklabel style={/pgf/number format/fixed,/pgf/number format/precision=2},
  nodes near coords,
  nodes near coords style={font=\scriptsize,/pgf/number format/fixed,/pgf/number format/precision=3}
]
\addplot+[fill=PrismPurple!50,draw=PrismPurple!85!black]
  coordinates {(Original,0.0626) (Rewritten,0.0587)};
\end{axis}
\end{tikzpicture}
\caption{Variation survives OPRO rewriting.}
\end{subfigure}
\caption{\textbf{Ordering structure transfers and is complementary to wording optimization.}
Left: source top orderings beat random for Qwen [+0.013,+0.051], Llama [+0.003,+0.118], and
Gemma [+0.008,+0.234]. Right: after 24 OPRO rewriting rounds, ordering-driven variation retains
94\% of its original standard deviation.}
\label{fig:transfer-baselines-v2}
\end{figure}
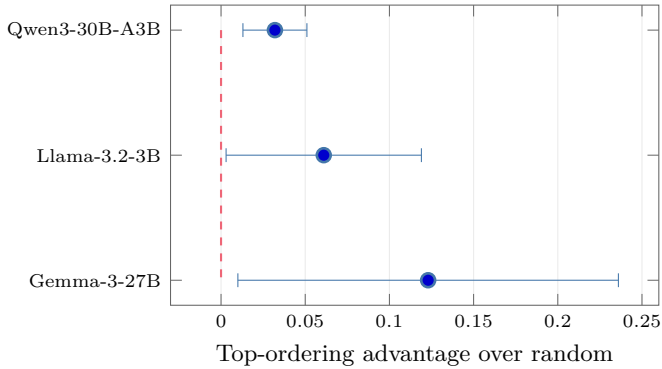
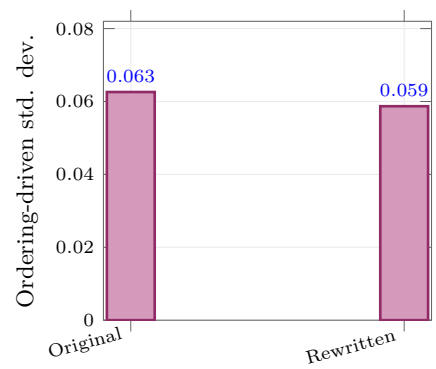

The wording baseline directly tests the objection that order sensitivity is only a symptom of poor
instruction text. OPRO raises held-out accuracy by +0.032, paired bootstrap 95\% interval
[+0.013,+0.052], yet the ordering standard deviation changes only from 0.0626 to 0.0587. Content
and order are therefore separate optimization coordinates.

We also compare OPRO with the evolutionary executor on the same 25-evaluation ordering budget,
using five shared runs. Final best-found accuracy is 0.500 [0.467,0.527] for evolution and 0.507
[0.480,0.540] for OPRO; their difference is \(-0.007\) [\(-0.053\),+0.040]. The comparison is a
tie at the tested budget, consistent with the near-zero-guidance forecast. The point is not that
the methods are identical, but that the pre-flight correctly avoids promising a directional win.

\begin{takeaway}
Ordering is a separate optimization axis from wording. Rewriting improves the average prompt but
does not flatten the ordering surface, while optimizer differences vanish in the forecast
near-random regime.
\end{takeaway}

%% file: sections/07_cross_domain.tex
\section{Cross-Domain Calibration}
\label{sec:cross-domain}

The protocol is also tested outside prompting. A scientific machine learning suite based on sparse
identification of nonlinear dynamics~\cite{brunton2016sindy} permutes six
preprocessing operations for sparse identification of nonlinear dynamics. All 720 orderings are
enumerated for each of six systems. Ordering changes fitness in every system; insertion has the
largest \(\rhoone\) in all six; and the evolutionary executor has more successful runs than
uniform sampling in all six, with pronounced advantages in four. The main miss is diagnostic:
Cayley distance is aligned with swaps and underestimates guidance on several insertion-shaped
pipeline landscapes.

A tabular neural architecture search study using a compact NAS-Bench-201
backend~\cite{dong2020nasbench201} registers 18 HIGH/LOW forecasts on fixed-operation
multiset slices before search. Thirteen are correct. CIFAR-100 and ImageNet16-120 are each 6/6;
all five misses occur on CIFAR-10-valid, and the three LOW misses are conservative because search
wins despite the cautious forecast. These studies test calibration, not a claim that PRISM is a
new pipeline learner or architecture-search algorithm.
The registered architecture forecasts are broken down by dataset and stratum in
\cref{fig:nas-calibration-v3}.

\begin{figure}[tbp]
\centering
\begin{tikzpicture}
\begin{axis}[
  prism axis,width=0.74\linewidth,height=0.34\linewidth,
  ybar,bar width=13pt,ymin=0,ymax=3.35,
  ylabel={Correct forecasts (of 3)},
  symbolic x coords={CIFAR-10,CIFAR-100,ImageNet16-120},xtick=data,
  x tick label style={font=\scriptsize},
  ytick={0,1,2,3},
  legend style={at={(0.5,1.02)},anchor=south,legend columns=2},
  nodes near coords,nodes near coords style={font=\scriptsize}
]
\addplot+[fill=PrismBlue!55,draw=PrismBlue!85!black]
  coordinates {(CIFAR-10,1) (CIFAR-100,3) (ImageNet16-120,3)};
\addlegendentry{HIGH: search predicted}
\addplot+[fill=PrismYellow!55,draw=PrismYellow!70!black]
  coordinates {(CIFAR-10,0) (CIFAR-100,3) (ImageNet16-120,3)};
\addlegendentry{LOW: sampling predicted}
\end{axis}
\end{tikzpicture}
\caption{\textbf{Registered NAS forecast calibration by dataset and stratum.} Each bar summarizes
three fixed-operation slices. Thirteen of 18 forecasts are correct. CIFAR-100 and ImageNet16-120
are correct in all cells; the CIFAR-10 LOW misses are conservative under-calls because search wins.}
\label{fig:nas-calibration-v3}
\end{figure}
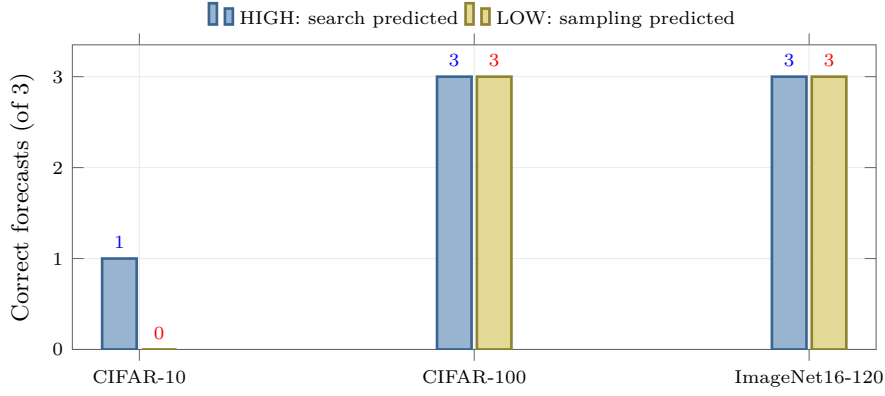

A \emph{forecast ledger} pairs every recorded call with its outcome, including corrections, ties,
and null results. The complete numerical record appears in \cref{tab:ledger-v2}; retaining misses
is what exposes the current distance-metric limitation rather than converting it into a post hoc
success.

Across the full program, recorded external-model inference expenditure remains below US\$20, and
the synthetic, neural, scientific-pipeline, and tabular architecture studies use CPU execution.
Exact tables, frozen pools, caches, and registered forecasts are retained so a wrong call remains
auditable rather than being rewritten as a successful post hoc explanation.

%% file: sections/08_limitations.tex
\section{Limitations}
\label{sec:limitations}

Exact neural ground truth stops at seven modules and exact language-model ground truth at six; the
eight-module study observes 222 of 40,320 orderings. The language evidence covers two mathematical
reasoning benchmarks and frozen model snapshots, not general generation, retrieval, coding, or
future model versions. Fixed question pools and graders also bound resolution: with 20 binary
questions, one answer moves fitness by 0.05 and produces dense ties.

Exact \(\fdc\) requires known optima. Sampled studies use the best frozen probe observation, which
can miss the relevant basin, while swap-aligned Cayley distance conservatively under-calls some
insertion-shaped pipelines. Precedence-aligned distance remains future work. The OPRO comparison
has five runs and supports a tie rather than fine equivalence.

\begin{takeaway}
The evidence supports calibrated decisions on measured landscapes, not universal optimizer
superiority or guaranteed transfer and calibration at larger \(n\).
\end{takeaway}

%% file: sections/09_conclusion.tex
\section{Conclusion}
\label{sec:conclusion}

Permutation-space difficulty follows structure rather than cardinality: a 5,040-order parity
landscape can be harder for elitist search than a matched synthetic space with roughly twenty
trillion orderings. PRISM therefore measures locality, global guidance, score variance, and tie
density before allocating budget. Its exhaustive instruction map spans 6.3\%--96.9\% accuracy,
yet dense optima and weak guidance still make uniform sampling the correct executor.

\textbf{Recommendations.} Define legal orderings, enumerate small spaces, verify evaluator
resolution, align distance with the move, and compare executors under one distinct-evaluation
budget. Register the forecast first and retain ties, censored runs, null results, and misses.

\textbf{Future work and further research.} Priorities are precedence- and adjacency-aligned
distances, uncertainty-aware probes, adaptive probe allocation, scalable surrogates, and
independent replication across tasks and model families. Language-agent pipelines, compiler and
database plans, scientific workflows, and laboratory automation offer low-cost replay; clinical,
security, financial, and autonomous applications require safety constraints and prospective human
oversight. The 24-domain roadmap in \cref{tab:applications-v2} separates opportunities from
validated evidence.

\begin{table}[tbp]
\centering
\scriptsize
\caption{Final summary of the paper's key findings, limitations, and recommended actions.}
\label{tab:summary-v3}
\begin{tabularx}{0.98\linewidth}{L{2.6cm}L{4.7cm}Y}
\toprule
\textbf{Focus} & \textbf{Key finding} & \textbf{Takeaway or action} \\
\midrule
Searchability & Space size does not predict difficulty; measured locality and guidance do & Run the PRISM pre-flight before committing the main budget \\
Executor selection & No executor dominates across guiding, near-random, and deceptive regimes & Select enumeration, sampling, evolution, a surrogate, or transfer conditionally \\
Exact falsification & On parity, elitist search hits 19/40 optima versus 30/40 for sampling & Treat a correct sampling forecast as a successful result \\
Instruction order & Six fixed modules span 6.3\%--96.9\% accuracy and show coherent position effects & Treat ordering as a distinct optimization coordinate \\
Transfer and baselines & Position structure transfers imperfectly and survives wording optimization & Revalidate rankings on each target and optimize wording and order separately \\
Cross-domain calibration & Scientific pipelines and NAS slices reveal both correct calls and distance-metric misses & Retain a forecast ledger and revise diagnostics from failures \\
Scope & Evidence is concentrated at small \(n\), fixed evaluators, and bounded task families & Replicate at larger scale with aligned metrics, uncertainty, and safety constraints \\
\bottomrule
\end{tabularx}
\end{table}

\begin{takeaway}
PRISM is a falsifiable protocol for deciding when to search, when to sample, and when to improve
the evaluator first; its value lies in calibrated choice rather than universal optimizer
superiority.
\end{takeaway}

%% file: sections/08_applications.tex
\section{Applications and Experimental Roadmap}
\label{sec:applications}

The protocol applies to fixed component sets whose order can change without changing their content.
\Cref{tab:applications-v2} maps 24 problem domains where ordering may affect accuracy, cost,
latency, safety, robustness, or resource use. Each study must define the legal space, verify score
resolution, register the forecast, and compare with sampling and the strongest domain baseline;
the entries are testable opportunities rather than claims of validated gains.

\input{sections/08_applications_table.tex}

\begin{takeaway}
High-impact use requires a legal ordering space, a resolvable evaluator, a registered forecast, and
equal-budget domain baselines.
\end{takeaway}

%% file: sections/08_applications_table.tex
\begin{table}[H]
\centering
\caption{Twenty-four problem domains in which PRISM can test ordering effects. Each entry defines a
legal permutation unit and an impact objective; it identifies a testable opportunity rather than a
validated win. New studies still require evaluator-resolution, safety, pre-flight, and equal-budget
baseline checks.}
\label{tab:applications-v2}
\renewcommand{\arraystretch}{1.24}
\begin{minipage}[t]{0.485\linewidth}
\normalsize
\begin{tabularx}{\linewidth}{L{2.35cm}Y}
\toprule
\textbf{Domain} & \textbf{Permuted unit and impact objective} \\
\midrule
Language-model prompting & Reasoning or tool instructions; accuracy, robustness, token use \\
Agent workflows & Retrieval, planning, action, checking; success, calls, latency \\
Compiler optimization & Legal passes; runtime, code size, energy \\
Neural architecture & Fixed blocks or edge operations; accuracy and efficiency \\
Scientific discovery & Preprocessing and identification stages; recovery and stability \\
Data engineering & Cleaning, joining, transformation stages; quality and cost \\
Clinical diagnostics & Test, triage, and review steps; safety, time, accuracy \\
Genomics & Quality control, alignment, and variant calling; sensitivity and runtime \\
Drug discovery & Screening, synthesis, and assay stages; yield and experimental cost \\
Manufacturing & Process operations; yield, defects, energy \\
Robotics & Perception, planning, and control modules; success and safety \\
Reinforcement learning & Options or curriculum tasks; sample efficiency and return \\
\bottomrule
\end{tabularx}
\end{minipage}\hfill
\begin{minipage}[t]{0.485\linewidth}
\normalsize
\begin{tabularx}{\linewidth}{L{2.35cm}Y}
\toprule
\textbf{Domain} & \textbf{Permuted unit and impact objective} \\
\midrule
Cybersecurity response & Detection, containment, recovery steps; risk and response time \\
Network service chains & Routing, filtering, inspection functions; throughput and latency \\
Cloud deployment & Build, test, deploy, migrate stages; downtime and cost \\
Database execution & Joins, filters, and aggregations; latency and resource use \\
Image processing & Denoising, normalization, segmentation; quality and accuracy \\
Signal processing & Filters and transforms; signal quality and latency \\
Remote sensing & Correction, fusion, classification stages; mapping accuracy \\
Financial risk & Cleaning, modelling, stress testing, reporting; stability and compliance \\
Supply chains & Sourcing, stocking, routing steps; cost and service level \\
Maintenance & Inspection, diagnosis, repair actions; reliability and downtime \\
Education & Lessons, examples, and exercises; mastery and retention \\
Laboratory automation & Preparation, mixing, measurement, quality control; reproducibility and time \\
\bottomrule
\end{tabularx}
\end{minipage}
\end{table}

%% file: sections/a_theory.tex
\section{Executor Guarantees and Scope}
\label{app:theory}

The protocol's empirical claims do not depend on a novel convergence theorem. The evolutionary
executor inherits the standard reachability argument for elitist evolutionary algorithms with an
irreducible mutation kernel~\cite{rudolph1994convergence}.
The executor cycle to which the following guarantees apply is shown in
\cref{fig:evolutionary-loop}.

\begin{figure}[tbp]
\centering
\includegraphics[width=0.90\linewidth]{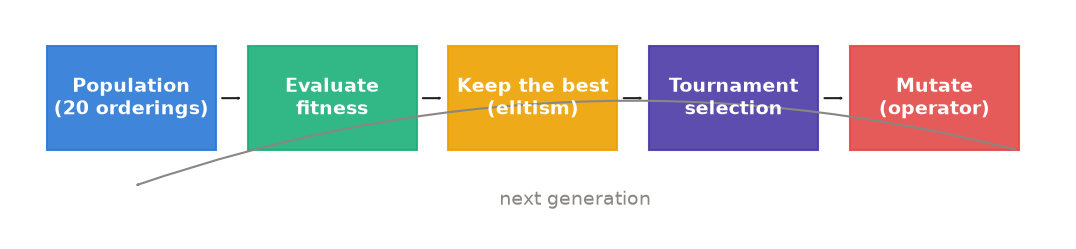}
\caption{\textbf{Evolutionary executor loop.} A population is evaluated, its best member is
retained, tournament selection chooses parents, and the selected mutation operator produces the
next generation. The convergence statements below apply to this executor, not to enumeration or
uniform sampling.}
\label{fig:evolutionary-loop}
\end{figure}

\begin{proposition}[Reachability]
\label{prop:reachability-v2}
Let \(\Sn\) be finite. If mutation has positive probability of proposing every transposition and
selection can choose every current population member with positive probability, then every
permutation is reachable from every population state in finitely many steps with positive
probability.
\end{proposition}

\begin{proof}
Every permutation can be transformed into every other by a finite sequence of transpositions.
The stated mutation and selection conditions assign positive probability to each step of such a
sequence. Their finite product is positive.
\end{proof}

\begin{theorem}[Almost-sure optimum discovery]
\label{thm:convergence-v2}
Under the conditions of \cref{prop:reachability-v2}, an elitist executor run without a finite
budget discovers a global optimum with probability one and never loses the best observed fitness.
\end{theorem}

\begin{proof}
From any non-optimal state there is a finite positive-probability path to an optimum. Finiteness
gives a positive lower bound on reaching an optimum over sufficiently long blocks. The probability
of avoiding all such blocks tends to zero. Elitism preserves the best fitness after discovery.
\end{proof}

The guarantee is qualitative. A direct worst-case bound scales with the number of permutations
and is therefore factorial; it does not explain the empirical \(n^3\log n\)-shaped fits observed
on matched synthetic landscapes through \(n=16\). Those fits are conditional evidence about
particular structures and moves, not a universal runtime theorem.

For aging replacement, the incumbent optimum can leave the active population. Nevertheless, the
best-ever record still converges almost surely under the same reachability conditions. This is why
the paper distinguishes active-population convergence from record convergence and makes all
finite-budget claims empirical.

%% file: sections/b_additional_results.tex
\section{Additional Statistical Record}
\label{app:statistics}

\begin{table}[tbp]
\centering
\small
\caption{Headline comparisons and uncertainty statements retained in the short paper. Bootstrap
entries are percentile intervals over the archived replay units; hit-rate entries use Wilson score
intervals for binomial proportions.}
\label{tab:ci-record-v2}
\begin{tabularx}{\linewidth}{L{3.4cm}L{3.4cm}L{3.4cm}Y}
\toprule
\textbf{Comparison} & \textbf{Method A} & \textbf{Method B} & \textbf{Interpretation} \\
\midrule
Parity optimum hits, 40 runs & elitist 19/40 [0.33,0.63] & random 30/40 [0.60,0.86] & no evolutionary advantage \\
Six-instruction evaluations to optimum & evolution 9.9 [7,13] & random 10.8 [8,14] & intervals overlap \\
MATH-500 final best-found accuracy & guided 0.407 [0.367,0.447] & random 0.397 [0.367,0.433] & random-competitive call holds \\
OPRO ordering search & evolution 0.500 [0.467,0.527] & OPRO 0.507 [0.480,0.540] & difference spans zero \\
Wording level change & +0.032 [0.013,0.052] & SD ratio 0.94 & wording helps; ordering remains \\
\bottomrule
\end{tabularx}
\end{table}

\subsection{Forecast ledger}

\begin{table}[tbp]
\centering
\scriptsize
\caption{Numerical forecast ledger. Correct, mixed, tie, and null outcomes remain in the record.}
\label{tab:ledger-v2}
\begin{tabularx}{0.98\linewidth}{L{2.25cm}L{3.25cm}Y L{1.35cm}}
\toprule
\textbf{Study} & \textbf{Recorded call} & \textbf{Outcome} & \textbf{Verdict} \\
\midrule
Six instructions & insertion; guiding structure & insertion signal; search/sampling intervals overlap & scoped \\
Eight modules & sampling competitive & both methods reach 1.0 at every budget & correct \\
Cross-family & source top orderings transfer & positive effect on all three targets; exact ranks vary & mixed \\
MATH-500 & sampling competitive & guided, evolutionary, and random intervals overlap & correct \\
Wording and search & order survives rewriting; search comparison uncertain & SD ratio 0.94; final difference spans zero & correct \\
Pipelines and NAS & cautious pre-search calls & pipeline distance under-calls; NAS 13/18 & partial \\
\bottomrule
\end{tabularx}
\end{table}

\subsection{Selector reliability}

At equal pre-flight budgets, pair-sampled \(\rhoone\) is more reliable than random-walk
autocorrelation or negative mean absolute fitness change. Averaged across typed synthetic and the
enumerated instruction landscape over 200 resamples per cell, selector accuracies are:
\begin{center}
\begin{tabular}{lccc}
\toprule
\textbf{Estimator} & \(B=50\) & \(B=100\) & \(B=500\) \\
\midrule
Pair-sampled \(\rhoone\) & \textbf{0.831} & \textbf{0.925} & \textbf{0.998} \\
Random-walk \(\rhoone\) & 0.734 & 0.848 & 0.992 \\
Negative mean absolute change & 0.827 & 0.915 & 0.991 \\
\bottomrule
\end{tabular}
\end{center}

\subsection{Pipeline and architecture-search details}

The six pipeline systems comprise damped and cubic oscillators, Van der Pol, Lotka--Volterra,
Lorenz-63, and R\"{o}ssler dynamics. Fitness ranges under order alone are 0.40--1.00, 0.70--1.00,
0.31--0.88, 0.00--1.00, 0.89--0.99, and 0.37--0.81, respectively. Evolutionary-versus-random
hit counts over 40 runs are 21/16, 34/30, 40/20, 40/37, 40/23, and 40/14.

The architecture-search forecast matrix contains three HIGH and three LOW slices per dataset.
CIFAR-100 and ImageNet16-120 are correct in all six cells. On CIFAR-10-valid, one HIGH call is
correct, two HIGH calls miss, and all three LOW calls are conservative misses because search wins.
Thus 13 of 18 registered forecasts are correct, and every LOW miss errs toward underusing rather
than overselling search.

\subsection{Completed experiment coverage}

\Cref{tab:coverage-v2} maps the complete research program into experiment families. The short main
text emphasizes the results that determine the current claim; the archive retains per-run data,
negative redesigns, corrections, and implementation details.

\begin{table}[tbp]
\centering
\scriptsize
\caption{Coverage of completed experiment families through August 3, 2026.}
\label{tab:coverage-v2}
\begin{tabularx}{\linewidth}{L{3.1cm}L{3.8cm}Y}
\toprule
\textbf{Experiment family} & \textbf{Included studies} & \textbf{Retained conclusion} \\
\midrule
Reproduction and benchmark repair & Historical neural toys; failed deep-stack redesign; residual,
seeded v4 benchmarks & Headlines reproduce with corrections; deterministic exact ground truth
replaces noisy best-ever fitness. \\
Operator and landscape science & Typed synthetic landscapes; deceptive objective; uniform and
adaptive portfolios & Operator matching is necessary on plateau-rich landscapes; portfolio removes
configuration failure at bounded observed cost; deception remains a limit. \\
Exact neural search tests & Enumerated XOR and parity landscapes from 120 to 5,040 orderings & Exact
hit rate and regret expose both success and the pivotal random-sampling falsification. \\
Diagnosis and method selection & Locality study; aging and restart replacement; precedence
surrogate; 40-seed audit; \(\rhoone\) ablation & Measurements select operators, replacement, and
representation; pair-sampled \(\rhoone\) is the most reliable tested low-budget selector. \\
Language-model ordering & Exhaustive and sampled instruction landscapes; size and family transfer;
guided evaluation; MATH-500; OPRO rewriting and search & Ordering has a large structured effect,
transfers imperfectly, survives wording optimization, and does not guarantee a search-speed win. \\
Cross-domain calibration & Six-system scientific-pipeline suite; 18 tabular NAS forecasts; two
lightweight RL ordering pilots & Scientific and NAS studies calibrate the protocol beyond
prompting; RL shows preliminary executor transfer but has not yet passed a registered pre-flight. \\
\bottomrule
\end{tabularx}
\end{table}

%% file: sections/c_reproducibility.tex
\section{Reproducibility, Resource Accounting, and Data Schema}
\label{app:reproducibility}

All externally hosted model evaluations use frozen question pools, deterministic decoding where
supported, canonical ordering identifiers, and append-only response records. Repeated orderings are
resolved from the archived record and do not count as new distinct evaluations. Recorded
external-model inference expenditure across the program remains below US\$20; the synthetic,
neural, scientific-pipeline, reinforcement-learning, and tabular architecture studies use CPU
execution.

\begin{table}[tbp]
\centering
\small
\caption{Resource accounting for externally hosted language-model studies. Expenditure is reported
for reproducibility and is not used as a performance metric.}
\begin{tabularx}{0.96\linewidth}{L{3.0cm}L{3.0cm}L{2.0cm}Y}
\toprule
\textbf{Study surface} & \textbf{Evaluation coverage} & \textbf{Observed cost} &
\textbf{Reproducibility record} \\
\midrule
Six instruction modules & \(720\times32\approx23{,}000\) responses & US\$3--5 & exhaustive
ordering-by-question table \\
Eight instruction modules & 23,805 model evaluations & US\$5.56 & fixed sampled pool with
per-response records \\
Cross-family transfer & 11,955 model evaluations & US\$0.66 & shared prompts, orderings, and
question set across three families \\
MATH-500 & 31,998 model evaluations & US\$4.43 & pilot-selected model, frozen pool, and registered
forecast \\
Wording and search baselines & 12,724 target evaluations & US\$2.68 & held-out scoring and shared
evaluation budgets \\
\bottomrule
\end{tabularx}
\end{table}

The data schema preserves one row per unique ordering with a stable study identifier, canonical
ordering key, component list, fitness statistic, sample provenance, evaluator version, exact-space
status, and resource record. Derived neighbourhood tables retain parent and child orderings, the
move operator, both scores, and the sampling seed; distance tables declare the metric and whether
the target is an exact optimum or a frozen proxy.

The canonical project repository, including the paper source, experiment summaries, data
manifests, and regeneration instructions, is available at \archiveurl.